\documentclass{article}

\usepackage{microtype}
\usepackage{graphicx}
\usepackage{subfigure}
\usepackage{booktabs} 

\usepackage{hyperref}

\usepackage[accepted]{icml2021}  

\icmltitlerunning{
Practical Privacy Filters and Odometers
}

\usepackage{amsfonts}
\usepackage{amssymb}
\usepackage{amsthm}
\usepackage{amssymb}
\usepackage{amsmath}

\newtheorem{theorem}{Theorem}
\newtheorem{corollary}{Corollary}

\def\E{\mathbb{E}}

\def\N{\mathbb{N}}

\def\cD{\mathcal{D}}

\def\cR{\mathcal{R}}
\def\cS{\mathcal{S}}

\def\edp{\epsilon_{\textrm{DP}}}
\def\erdp{\epsilon_{\textrm{RDP}}}

\def\PASS{$\textrm{PASS}$}

\newcommand\mPASS{\ensuremath{\textrm{PASS}}}
\newcommand\truncf{\ensuremath{\textrm{trunc}_{\leq f}}}
\newcommand\loss{\ensuremath{\textrm{Loss}}}
\newcommand\tloss{\ensuremath{\widetilde{\textrm{Loss}}}}

\begin{document}

\twocolumn[
\icmltitle{
Practical Privacy Filters and Odometers with R\'enyi Differential Privacy \\
and Applications to Differentially Private Deep Learning
}

\icmlsetsymbol{equal}{*}

\begin{icmlauthorlist}
\icmlauthor{Mathias L\'ecuyer}{msrny}
\end{icmlauthorlist}

\icmlaffiliation{msrny}{Microsoft Research, New York City}

\icmlcorrespondingauthor{Mathias Lecuyer}{mathias.lecuyer@gmail.com}

\vskip 0.3in
]

\printAffiliationsAndNotice{}  

\begin{abstract}
Differential Privacy (DP) is the leading approach to privacy preserving  deep learning.
As such, there are multiple efforts to provide drop-in integration of DP into popular frameworks.
These efforts, which add noise to each gradient computation to make it DP, rely on composition theorems to bound the total privacy loss incurred over this sequence of DP computations.

However, existing composition theorems present a tension between efficiency and flexibility.
Most theorems require all computations in the sequence to have a predefined DP parameter, called the privacy budget. This prevents the design of training algorithms that adapt the privacy budget on the fly, or that terminate early to reduce the total privacy loss.
Alternatively, the few existing composition results for adaptive privacy budgets provide complex bounds on the privacy loss, with constants too large to be practical.

In this paper, we study DP composition under adaptive privacy budgets through the lens of R\'enyi Differential Privacy, proving a simpler composition theorem with smaller constants, making it practical enough to use in algorithm design.
We demonstrate two applications of this theorem for DP deep learning: adapting the noise or batch size online to improve a model's accuracy within a fixed total privacy loss, and stopping early when fine-tuning a model to reduce total privacy loss.
\end{abstract}

\section{Introduction}

Performing statistical analyses or training machine learning models have been shown to leak information about individual data points used in the process.
Differential Privacy (DP) is the leading approach to prevent such privacy leakage, by adding noise to any computation performed on the data such that the privacy loss resulting their result is provably bounded.
Due to DP's strong privacy guarantees, multiple efforts exist to integrate it into libraries, including for statistical queries \cite{google-dp,opendp}, machine learning \cite{diffprivlib}, and deep learning \cite{tensorflow-privacy,opacus}.
The DP analysis in these libraries rely on composing the results of a sequence of DP computations, consisting for instance in a sequences of statistical queries, or a sequence of mini-batches of gradients in deep learning.
Each computation in the sequence can be seen as consuming some privacy budget, that depends on the size of the noise added to the computation's result.
Composition theorems are then applied to bound to the total privacy budget, or privacy loss, used over the sequence.

However efficient DP composition theorems require that all computations in the entire sequence have a privacy budget fixed a priori.
This creates a mismatch between the theory, and practical requirements.
Indeed, users who interact with a dataset through statistical queries or develop machine learning models will typically want to adapt the next query or model's DP budget allocation based on previous results, resulting in an adaptive behavior.
At the algorithm level this mismatch is visible in the APIs developed for DP deep learning, which track DP computations as they happen during training. Users are responsible for avoiding early stopping or other types of parameter adaptation.
This restriction reduces the design space for DP deep learning algorithms, ruling out early stopping to reduce overall privacy loss, DP budgets that adapt during training to improve the final performance, or online architecture search to adapt the complexity of the model to the data within DP constraints.

The main study of adaptive composition \cite{10.5555/3157096.3157312} defines two relevant interaction models that yield different asymptotic rates for composition theorems.
The {\em privacy filter} model enforces an a priori bound on the total privacy loss, and can reach the same rate as non-adaptive composition.
The {\em privacy odometer} model is more costly, but more flexible as it provides a running upper-bound on the sequence of DP computations performed so far, which can stop at any time. However there is a small penalty to pay in asymptotic rate.
\cite{10.5555/3157096.3157312} provides composition theorems with optimal rate for both models, but the bounds are complex and with large constants, making them impractical.
A recent paper, \cite{feldman2020individual}, shows improved filters but unsatisfactory odometers.

In this paper, we study both privacy filters and odometers through the lens of R\'enyi Differential Privacy (RDP).
RDP in an analysis tool for DP composition that enables intuitive privacy loss tracking and provides efficient composition results under non-adaptive privacy budget. It has become a key building block in DP deep learning libraries \cite{tensorflow-privacy,opacus}.
We show that RDP can be used to provide privacy filters without incurring any cost for adaptivity.
We build on this result to construct privacy odometers that remove limitations of previously known results, in particular in regimes relevant for DP deep learning.

We demonstrate two applications of these theoretical results when training DP deep learning models.
First, we leverage privacy filters to show that even basic policies to adapt on the fly the noise added to gradient computations, as well as the batch size, yield improved accuracy of more than $2.5\%$ within a fixed privacy loss bound on CIFAR-10.
Second, we demonstrate that in a fine-tuning scenario our privacy odometer enables large reductions in total privacy budget using adaptivity and early-stopping.

\section{Privacy Background}
\label{sec:privacy-background}

This section provides the necessary background on DP and composition, and introduces some notation.
Specific results on adaptive composition are discussed and compared to our contributions throughout the paper.

\subsection{Differential Privacy}

\paragraph{Definition.} We first recall the definition of (approximate) Differential Privacy \cite{dp-book}. A randomized mechanisms $M : \cD \rightarrow \cR$ satisfies $(\epsilon, \delta)$-DP when, for any neighboring datasets $D$ and $D'$ in $\cD$, and for any $\cS \subset \cR$:
\[
  P(M(D) \in \cS) \leq e^\epsilon P(M(D') \in \cS) + \delta .
\]
This definition depends on the notion of neighboring datasets, which can be domain specific but is typically chosen as the addition or removal of one data point. Since the neighboring relation is symmetric, the DP inequality also is.
A smaller privacy budget $\epsilon>0$ implies stronger guarantees, as one draw from the mechanism's output gives little information about whether it ran on $D$ or $D'$.
When $\delta = 0$, the guarantee is called pure DP, while $\delta > 0$ is a relaxation of pure DP called approximate DP.

A useful way of rephrasing the DP guarantee is to define the privacy loss as, for $m \in \cR$:
\[
  \loss(m) = \log \Big( \frac{P\big( M(D) = m \big)}{P\big( M(D') = m \big)} \Big) .
\]
If with probability at least $(1-\delta)$ over $m \sim M(D)$:
\begin{equation}
  \label{eq:dp-as-privacy-loss-bound}
  | \loss(m) | \leq \epsilon ,
\end{equation}
then the mechanism $M$ is $(\epsilon, \delta)$-DP.
We can see in this formulation that $\epsilon$ is a bound on the privacy loss.
Approximate DP allows for a $\delta > 0$ failure probability on that bound.

\paragraph{Composition.} Composition is a core DP property. Composition theorems bound the total privacy loss incurred over a sequence of DP computations. This is useful both as a low level tool for algorithm design, and to account for the privacy loss of repeated data analyses performed on the same dataset.
Basic composition states that a sequence $M_i$ of $(\epsilon_i, \delta_i)$-DP mechanisms is $(\sum \epsilon_i, \sum \delta_i)$-DP.
While not efficient, it is simple and provides an intuitive notion of additive privacy budget consumed by DP computations.

Strong composition provides a tighter bound on privacy loss, showing that the same sequence is $(\epsilon_g, \delta_g)$-DP with $\epsilon_g = \sqrt{2 \log(\frac{1}{\hat\delta}) \sum \epsilon_i^2} + \sum \epsilon_i(e^\epsilon_i - 1)$ and $\delta_g = \hat\delta + \sum \delta_i$.
This result comes with its share of drawbacks though. Each mechanism $M_i$ comes with its own trade-off between $\epsilon_i$ and $\delta_i$, and there is no practical way to track and optimally compose under these trade-offs.
This has led to multiple special purpose composition theorems, such as the moment accountant for DP deep learning \cite{abadi2016deep}.
Strong composition results also require the sequence of DP parameters to be fixed in advance.

\subsection{R\'enyi Differential Privacy}
RDP \cite{8049725} is also a relaxation of pure DP that always implies approximate DP, though the converse in not always true.

\paragraph{R\'enyi Differential Privacy (RDP).} RDP expresses its privacy guarantee using the R\'enyi divergence:
\[
  D_\alpha(P || Q) \triangleq \frac{1}{\alpha - 1} \log \E_{x \sim Q} \Big( \frac{P(x)}{Q(x)} \Big)^\alpha ,
\]
where $\alpha \in ]1, +\infty]$  is the order of the divergence.
A randomized mechanism $M$ is $(\alpha, \epsilon)$-RDP if:
\[
  D_\alpha(M(D) || M(D')) \leq \epsilon .
\]

\paragraph{RDP composition.} A sequence of $(\alpha, \epsilon_i)$-RDP mechanisms is $(\alpha, \sum \epsilon_i)$-RDP.
We can see that RDP re-introduces an intuitive additive privacy budget $\epsilon$, at each order $\alpha$.

Tracking the RDP budget over multiple orders enables one to account for the specific trade-off curves of each specific mechanism in a sequence of computations. Indeed, for an $ (\alpha, \epsilon)$-RDP mechanism:
\begin{equation}
  \label{rdp-to-dp}
  P_{m \sim M(D)}\Big(|\loss(m)| > \epsilon + \frac{\log(1/\delta)}{\alpha-1}\Big) \leq \delta ,
\end{equation}
which in turn implies $(\epsilon + \frac{\log(1/\delta)}{\alpha-1}, \delta)$-DP.
Although RDP composition is similar to strong composition in the worst case, tracking privacy loss over the curve of R\'enyi orders often yields smaller privacy loss bounds in practice.

\section{Privacy Filters}
We first focus on privacy filters, which authorize the composition of adaptive DP budgets within an upper-bound on the privacy loss fixed a priori.
That is, the sequence of computations' budgets is not known in advance, and can depend on previous results, and the filter ensures that the upper-bound is always valid.

\subsection{The Privacy Filter Setup}
Complex sequences of interactions with the data, such as the adaptive privacy budget interactions we focus on in this work, can be challenging to express precisely.
Following \cite{10.5555/3157096.3157312}, we express them in the form of an algorithm that describes the power and constraints of an analyst, or adversary since it models worst case behavior, interacting with the data.
Algorithm \ref{filter-comp} shows this interaction setup for a privacy filter.

Specifically, an adversary $A$ will interact for $k$ rounds ---$k$ does not appear in the bounds and can be chosen arbitrarily large---, in one of two ``neighboring worlds'' indexed by $b \in \{0, 1\}$.
Typically, the two worlds will represent two neighboring datasets with an observation added or removed, but this formulation allows for more flexibility by letting $A$ chose two neighboring datasets at each round, as well as the DP mechanism to use and its RDP parameters. Each decision can be made conditioned on the previous results.

Our RDP filter $\textrm{FILT}_{\alpha, \erdp}$ of order $\alpha$ can \PASS~on a computation at any time, setting $\epsilon_i = 0$ and returning $\perp$ to $A$.
Sections \ref{sec:filter-anaysis} and \ref{sec:filter-renyi-curve} describe how to instantiate this filter to provide RDP and DP guarantees.
The final output of the algorithm is a view $V^b=(v_1, \ldots, v_k)$ of all results from the computations.
With a slight abuse of notation, we call $V^b_i$ the distribution over possible outcomes for the $i^{th}$ query, and $v_i \sim V^b_i$ an observed realization.

This way, we can explicitly write the dependencies of the $i^{th}$ query's distribution as $V^b_i(v_i | v_{\leq i-1})$, where $v_{\leq i}$ is the view of all realized outcomes up to, and including, the $i^{th}$ query.
Similarly, the joint distribution over possible views is $V^b(v_{\leq n})$.
The budget of the $i^{th}$ query, $\epsilon_i$ depends only on all queries up to the previous one, which we note explicitly as $\epsilon_i(v_{\leq i-1})$.
DP seeks to bound the information about $b$ that $A$ can learn from the results of its DP computations.
Following Equation \ref{eq:dp-as-privacy-loss-bound}, we want to bound with high probability under $v \sim V^0$ the privacy loss incurred by the sequence of computations $|\loss(v)| = \big| \log \big( \frac{V^0(v_{\leq n})}{V^1(v_{\leq n})} \big) \big|$.

\begin{algorithm}[H]
  \caption{$\textrm{FilterComp}(A,k,b, \alpha; \textrm{FILT}_{\alpha, \erdp})$}
  \label{filter-comp}
\begin{algorithmic}
  \FOR {$i = 1, \ldots, k$}
    \STATE $A = A(v_1, \ldots, v_{i-1})$ gives neighboring $D_i^{0}, D_i^{1}, \epsilon_i, \textrm{and~} M_i: \cD \rightarrow \cR$ an $(\alpha, \epsilon_i)$-RDP mechanism\;
    \IF {$\textrm{FILT}_{\alpha, \erdp}(\epsilon_1, \ldots, \epsilon_{i}, 0, \ldots, 0) = \textrm{PASS}$}
      \STATE $\epsilon_i = 0$\;
      \STATE $A$ receives $v_i = V^b_i = \perp$\;
    \ELSE
      \STATE $A$ receives $v_i \sim V^b_i = M_i(D_i^{b})$\;
    \ENDIF
  \ENDFOR
  \OUTPUT view $V^b=(v_1, \ldots, v_k)$.
\end{algorithmic}
\end{algorithm}

\subsection{R\'enyi Differential Privacy Analysis}
\label{sec:filter-anaysis}

We first analyse the filter $\textrm{FILT}_{\alpha, \erdp}(\epsilon_1, \ldots, \epsilon_{i}, 0, \ldots, 0)$ that will $\textrm{PASS}$ when:
\begin{equation}
  \label{priv-filter}
  \sum_i \epsilon_i > \erdp .
\end{equation}

The filter $\textrm{PASS}$es when the new query would cause the sum of all budgets used so far to go over $\erdp$.
We next show that this filter ensures that Algorithm \ref{filter-comp} is $(\alpha, \erdp)$-RDP.

\begin{theorem}[RDP Filter]
\label{prop:filter-composition}
The interaction mechanism from Algorithm \ref{filter-comp}, instantiated with the filter from Equation \ref{priv-filter}, is $(\alpha, \erdp)$-RDP. That is: $D_\alpha(V^0 || V^1) \leq \erdp$.
\end{theorem}
\begin{proof}
We start by writing the integral for $D_\alpha(V^0 || V^1)$ as:
\begin{align*}
  & e^{(\alpha - 1) D_\alpha( V^0 || V^1 )} \\
  &  = \int_{\cR_1 \ldots \cR_n} \big( V^0(v_{\leq n}) \big)^\alpha \big( V^1(v_{\leq n}) \big)^{1-\alpha} dv_1 \ldots dv_n \\
  &  = \int_{\cR_1} \big( V_1^0(v_1) \big)^\alpha \big( V_1^1(v_1) \big)^{1-\alpha} \ldots \\
  &    \int_{\cR_n} \big( V_n^0(v_n | v_{\leq n-1}) \big)^\alpha \big( V_n^1(v_n | v_{\leq n-1}) \big)^{1-\alpha} dv_n \ldots dv_1 .
\end{align*}

Focusing on the inner-most integral, we have that:
\begin{align*}
  & \int_{\cR_n} \big( V_n^0(v_n | v_{\leq n-1}) \big)^\alpha \big( V_n^1(v_n | v_{\leq n-1}) \big)^{1-\alpha} dv_n \\
  & \leq e^{(\alpha - 1) \epsilon_n(v_{\leq n-1})} \leq e^{(\alpha - 1)  \big(\erdp - \sum_{i=1}^{n-1} \epsilon_i(v_{\leq i-1}) \big)} ,
\end{align*}
  where the first inequality follows because $M_n$ is $\big( \alpha, \epsilon_n(v_{\leq n-1}) \big)$-RDP, and the second because the privacy filter enforces that $\sum_{i=1}^{n} \epsilon_i \leq \erdp \Rightarrow \epsilon_n \leq \erdp - \sum_{i=1}^{n-1} \epsilon_i$, since $\epsilon_i \geq 0$.

  Note that now, the inner most integral is bounded by a function of $v_{\leq n-2}$ and not  $v_{n-1}$. Thus, we can write:
\begin{align*}
  & e^{(\alpha - 1) D_\alpha( V^0 || V^1 )} \\
  & \leq \int_{\cR_1} \big( V_1^0(v_1) \big)^\alpha \big( V_1^1(v_1) \big)^{1-\alpha} \ldots \\
       & \int_{\cR_{n-2}} \big( V_{n-2}^0(v_{n-2} | v_{\leq n-3}) \big)^\alpha \big( V_{n-2}^1(v_{n-2} | v_{\leq n-3}) \big)^{1-\alpha} \\
       & \Big\{ e^{(\alpha - 1)  \big(\erdp - \sum_{i=1}^{n-1} \epsilon_i(v_{\leq i-1}) \big)} \\
       & \int_{\cR_{n-1}} \big( V_{n-1}^0(v_{n-1} | v_{\leq n-2}) \big)^\alpha \big( V_{n-1}^1(v_{n-1} | v_{\leq n-2}) \big)^{1-\alpha} \\
       & dv_{n-1} \Big\} dv_{n-2} \ldots dv_1 \\
  & \leq \int_{\cR_1} \big( V_1^0(v_1) \big)^\alpha \big( V_1^1(v_1) \big)^{1-\alpha} \ldots \\
       & \int_{\cR_{n-2}} \big( V_{n-2}^0(v_{n-2} | v_{\leq n-3}) \big)^\alpha \big( V_{n-2}^1(v_{n-2} | v_{\leq n-3}) \big)^{1-\alpha} \\
       & \Big\{ e^{(\alpha - 1)  \big(\erdp - \sum_{i=1}^{n-1} \epsilon_i(v_{\leq i-1}) \big)}
                e^{(\alpha - 1)  \epsilon_{n-1}(v_{\leq n-2})} \Big\} \\
       & dv_{n-2} \ldots dv_1 \\
  & = \int_{\cR_1} \big( V_1^0(v_1) \big)^\alpha \big( V_1^1(v_1) \big)^{1-\alpha} \ldots \\
       & \int_{\cR_{n-2}} \big( V_{n-2}^0(v_{n-2} | v_{\leq n-3}) \big)^\alpha \big( V_{n-2}^1(v_{n-2} | v_{\leq n-3}) \big)^{1-\alpha} \\
       & \Big\{ e^{(\alpha - 1)  \big(\erdp - \sum_{i=1}^{n-2} \epsilon_i(v_{\leq i-1}) \big)} \Big\} dv_{n-2} \ldots dv_1 \\
  &  \leq e^{(\alpha - 1)  \erdp} ,
\end{align*}
where the first inequality takes the previous bound out of the integral as it is a constant w.r.t. $v_{n-1}$, the second inequality comes from $M_{n-1}$ being $\big( \alpha, \epsilon_{n-1}(v_{\leq n-2}) \big)$-RDP, and the equality combines the two exponentials and cancels $\epsilon_{n-1}$ in the sum. As we can see, the term in the brackets is now only a function of $v_{\leq n-3}$ and not  $v_{n-2}$.
The last inequality recursively applies the same reasoning to each integral, removing all the $\epsilon_i$ in turn. Taking the log of each side of the global inequality concludes the proof.
\end{proof}

\subsection{Composing Over the R\'enyi Curve}
\label{sec:filter-renyi-curve}

A powerful feature of RDP is to be able to compose mechanisms over their whole $(\alpha, \epsilon)$ curve.
Since for many mechanisms this curve is not linear, it is unclear in advance which $\alpha$ would yield the best final $(\epsilon, \delta)$-DP guarantee through Equation \ref{rdp-to-dp}.
This composition ``over all $\alpha$s'' is important, as it enables RDP's strong composition results.

We can extend our privacy filter and Theorem \ref{prop:filter-composition} to this case. Consider a (possibly infinite) sequence of $\alpha$ values to track, noted $\Lambda$, and an upper bound $\erdp(\alpha)$ for each of these values.
In Algorithm \ref{filter-comp}, the mechanism $M_i$ is now $\big(\alpha, \epsilon_i(\alpha)\big)$-RDP for  $\alpha \in \Lambda$.
The filter $\textrm{FILT}_{\alpha, \erdp(\alpha)}\big(\epsilon_1(\alpha), \ldots, \epsilon_{i}(\alpha), 0, \ldots, 0\big)$ will now \PASS~when:
\begin{equation}
  \label{priv-filter-alpha}
  \forall \alpha \in \Lambda, \sum_i \epsilon_i(\alpha) > \erdp(\alpha) .
\end{equation}

\begin{corollary}[RDP Filter over the R\'enyi Curve]
\label{prop:filter-renyi-curve-composition}
The interaction mechanism from Algorithm \ref{filter-comp}, instantiated with the filter from Equation \ref{priv-filter-alpha} is such that $\exists \alpha: D_\alpha(V^0 || V^1) \leq \erdp(\alpha)$.
\end{corollary}
\begin{proof}
We argue by contradiction that $\exists \alpha: \sum_i \epsilon_i(\alpha) \leq \erdp(\alpha)$. Assume that was not the case, and consider $j = \arg\max_i (\max_\alpha \epsilon_i(\alpha)\neq 0)$ the maximal index for which the filter did not \PASS. Then $\sum_i \epsilon_i(\alpha) > \erdp(\alpha)$, the filter \PASS ed and $\forall \alpha: \epsilon_i(\alpha) = 0$, which is a contradiction.
Applying Theorem~\ref{prop:filter-composition} to this $\alpha$ concludes the proof.
\end{proof}

\subsection{$(\epsilon, \delta)$-DP Implications}

A key application of Corollary~\ref{prop:filter-renyi-curve-composition} is to build a privacy filter that enforces an $(\edp, \delta)$-DP guarantee, and leverages RDP to ensure strong composition.
Intuitively, for each $\alpha$ the corresponding $\erdp(\alpha)$ is set to the value which, for this $\alpha$, implies the desired $(\edp, \delta)$-DP target.
Formally, based on Equation \ref{rdp-to-dp} we set the filter from Equation \ref{priv-filter-alpha} to $\forall \alpha \in \Lambda: \erdp(\alpha) = \edp - \frac{\log(1/\delta)}{\alpha-1}$.
Applying Corollary~\ref{prop:filter-renyi-curve-composition} followed by Equation \ref{rdp-to-dp} shows that under this filter, Algorithm \ref{filter-comp} is $(\edp, \delta)$-DP.

This is significant, as this construction yields strong DP composition with a simple additive (RDP) privacy filter.
There is no overhead compared to non-adaptive RDP, which yields tighter results than DP strong composition in many practical settings, such as the composition of Gaussian mechanisms ubiquitous in DP deep learning.
On the contrary, Theorem 5.1 from \cite{10.5555/3157096.3157312} is restricted to DP strong composition, and pays a significant multiplicative constant.
The result from Theorem \ref{prop:filter-composition} already appears in \cite{feldman2020individual} as a corollary to a more general result.
Here we show a proof closer to that of RDP composition, that we will re-use in subsequent results.
We also explicitly show how to implement filters over the R\'enyi curve in \S\ref{sec:filter-renyi-curve}.

\section{Privacy Odometers}
\label{sec:odometers}

Privacy filters are useful to support sequences of computations with adaptive privacy budgets, under a fixed total privacy loss.
In many cases however, we are interested in a more flexible setup in which we can also stop at any time, and bound the privacy loss of the realized sequence.
Examples of such use-cases include early stopping when training or fine-tuning deep learning models, ``accuracy first'' machine learning \cite{ligett2017accuracy}, or an analyst interacting with the data until a particular question is answered.

\subsection{The Privacy Odometer Setup}
Algorithm \ref{odometer-comp} formalizes our privacy odometer setup. The adversary $A$ can now interact with the data in an unrestricted fashion.
For the given RDP order $\alpha$, the odometer returns the view and the RDP parameter $\erdp^{tot}$, which must be a running upper-bound on $D_\alpha \big( V^0(v_{\leq i}) || V^1(v_{\leq i}) \big)$ since the interaction can stop at any time.

\begin{algorithm}[H]
  \caption{$\textrm{OdometerComp}(A,k,b, \alpha, \erdp^{f \in \N_1})$}
  \label{odometer-comp}
\begin{algorithmic}
  \FOR {$i = 1, \ldots, k$}
    \STATE $A = A(v_1, \ldots, v_{i-1})$\;
    \STATE $A$ gives neighboring $D_i^{0}, D_i^{1}, \epsilon_i, \textrm{and~} M_i: \cD \rightarrow \cR$ an $(\alpha, \epsilon_i)$-RDP mechanism\;
    \STATE $f = \arg\min_g \{ \textrm{FILT}_{\alpha, \erdp^g}(\epsilon_1, \ldots, \epsilon_i, 0, \ldots, 0) \neq \mPASS \}$\;
    \STATE $A$ receives $v_i \sim V^b_i = M_i(D_i^{b})$\;
    \ENDFOR
  \OUTPUT view $V^b=(v_1, \ldots, v_k)$, stopping filter $f$, and spent budget $\erdp^{tot} = \erdp^{f}$.
\end{algorithmic}
\end{algorithm}

Analyzing such an unbounded interaction directly with RDP is challenging, as the very RDP parameters can change at each step based on previous realizations of the random variable, and there is no limit enforced on their sum.
We sidestep this challenge by initiating the odometer with a sequence of $\erdp^{f}$, with $f \in \N_1$ a strictly positive integer.
After each new RDP computation, we compute the smallest $f$ such that a filter $\textrm{FILT}_{\alpha, \erdp^f}$ initiated with budget $\erdp^{f}$ would not \PASS.
We show how to analyse the privacy guarantees of this construct in sections \ref{sec:odometer-anaysis} and \ref{sec:dp-odometer}, before extending it to support RDP accounting over the R\'enyi curve in section \ref{sec:odometer-renyi-curve}.
Finally, section \ref{sec:odometer-instantiation} discusses how to instantiate $\erdp^{f \in \N_1}$ to get efficient odometers.

\subsection{R\'enyi Differential Privacy Analysis}
\label{sec:odometer-anaysis}

To analyse Algorithm \ref{odometer-comp}, we define a truncation operator $\truncf$, which changes the behavior of an adversary $A$ to match that of a privacy filter of size $\erdp^f$.
Intuitively, $\truncf(A)$ follows $A$ exactly until $A$ outputs an RDP mechanism that would cause a switch to the $(f+1)^{th}$ filter.
When this happens, $\truncf(A)$ stops interacting until the end of the $k$ queries.
More precisely, at round $i$ $\truncf(A)$ outputs the same $D_i^{0}, D_i^{1}$, $\epsilon_i$, and $M_i$ as $A$ as long as $\arg\min_g \{ \textrm{FILT}_{\alpha, \erdp^g}(\epsilon_0, \ldots, \epsilon_i, 0, \ldots, 0) \neq \mPASS \} \leq f$. Otherwise,  $\truncf(A)$ returns $\epsilon_i = 0$ and $M_i: x \rightarrow \perp$.

We similarly note $\truncf(V^b)$ the distribution over possible views when adversary $\truncf(A)$ interacts following Algorithm \ref{odometer-comp}, and $\truncf(v)$ a corresponding truncation of a view generated by adversary $A$.
This truncated adversary behaves like a privacy filter, yielding the following result.

\begin{corollary}[Truncated RDP Odometer]
  \label{prop:truncated-rdp-odometer}
  The interaction mechanism from Algorithm \ref{odometer-comp}, is such that $D_\alpha \big( \truncf(V^0) || \truncf(V^1) \big) \leq \erdp^f$.
\end{corollary}
\begin{proof}
  The interaction from Algorithm \ref{odometer-comp} with $\truncf(A)$ is identical to that of Algorithm \ref{filter-comp} with $\truncf(A)$ initiated with $\textrm{FILT}_{\alpha, \erdp^f}$, since the filter would never \PASS.
  The statement follows from Theorem \ref{prop:filter-composition}.
\end{proof}

We note that the odometer from \cite{feldman2020individual} (Proposition 5.2) is closer to this truncated odometer than to the odometer from \cite{10.5555/3157096.3157312}. Indeed, the guarantee it gives is for $k$ stacked privacy filters, where $k$ is fixed in advance.
If applied with an adaptive $k$ (e.g. early stopping) the result would violate the lower bound proven in \cite{10.5555/3157096.3157312}, Theorem 6.1.
We next show how to provide odometer semantics with RDP composition bounds.

\subsection{$(\epsilon, \delta)$-DP Implications}
\label{sec:dp-odometer}

We are now ready to bound the privacy loss of Algorithm \ref{odometer-comp}.

\begin{theorem}[DP Odometer]
\label{prop:dp-odometer}
  For a given outcome $v$ and $f$ of the interaction described in Algorithm \ref{odometer-comp}, we have
  that $P\big(|\loss(v)| > \erdp^f + \frac{\log(2f^2/\delta)}{\alpha-1}\big) \leq \delta$.

  This implies that the interaction is $(\erdp^f + \frac{\log(2f^2/\delta)}{\alpha-1}, \delta)$-DP.
\end{theorem}
\begin{proof}
  We first show an important relationship between the interaction of Algorithm \ref{odometer-comp} and its truncated version. Denote $\mPASS_f$ the event that $\truncf(V^0)$ stopped, that is $\truncf(v)$ includes at least one $\perp$. Denote $F$ the random variable from which $f$ is drawn (its randomness comes from the randomness in both $A$ and $V$).

  Further note $\tloss(\tilde{v}) = \log\big( \frac{P(\truncf(V^0) = \tilde{v})}{P(\truncf(V^1) = \tilde{v})} \big)$ the privacy loss of an outcome $\tilde{v} \sim \truncf(V^0)$ in the truncated version of Algorithm \ref{odometer-comp}.
  Because truncation at $f$ does not change anything when $F \leq f$, we have that:
\begin{align*}
  & P(V^b = v \big| F \leq f) = P(\truncf(V^b) = v \big| \lnot \mPASS_f) \\
  \Rightarrow & P\big(|\loss(v)| > c \big| F \leq f \big) = P\big(|\tloss(v)| > c \big| \lnot \mPASS_f \big) \\
  \Rightarrow & P\big(|\loss(v)| > c, F \leq f \big) = P\big(|\tloss(v)| > c, \lnot \mPASS_f \big) ,
\end{align*}
 where the last inequality uses the fast that $P(F \leq f) = P(\lnot \mPASS_f)$ by definition of the truncation operator and $\mPASS_f$ event.

Second, we bound the privacy loss of the non-truncated interaction as follows:
\begin{align*}
  & P\big(|\loss(v)| > \erdp^f + \frac{\log(2f^2/\delta)}{\alpha-1}\big) \\
  & = \sum_f P\big(|\loss(v)| > \erdp^f + \frac{\log(2f^2/\delta)}{\alpha-1}, F=f \big) \\
  & \leq \sum_f P\big(|\loss(v)| > \erdp^f + \frac{\log(2f^2/\delta)}{\alpha-1}, F \leq f \big) \\
  & = \sum_f P\big(|\tloss(v)| > \erdp^f + \frac{\log(2f^2/\delta)}{\alpha-1}, \lnot \mPASS_f \big) \\
  & \leq \sum_f P\big(|\tloss(v)| > \erdp^f + \frac{\log(2f^2/\delta)}{\alpha-1} \big) \\
  & \leq \sum_f \frac{\delta}{2f^2} \leq \delta .
\end{align*}
  Where the second equality follows from the previous step, and the penultimate inequality applies Corollary \ref{prop:truncated-rdp-odometer} with $\delta = \frac{\delta}{2f^2}$, and Equation \ref{rdp-to-dp}.
\end{proof}

\subsection{Composing Over the R\'enyi Curve}
\label{sec:odometer-renyi-curve}

Theorem \ref{prop:dp-odometer} gives us an RDP based bound on the privacy loss, with a penalty for providing a running upper-bound over the sequence of $\erdp^f$.
To fully leverage RDP strong composition, we once again need to track the RDP budget over the R\'enyi curve. Applying a union bound shows that, for all $\alpha \in \Lambda$:
\begin{equation}
  \label{eq:dp-odom-curve}
  P\big(|\loss(v)| > \erdp^f + \frac{\log(|\Lambda| 2f^2/\delta)}{\alpha-1}\big) \leq \delta .
\end{equation}

Concretely, we track the budget spent by $A$ in Algorithm \ref{odometer-comp} over a set of RDP orders $\alpha \in \Lambda$, denoted $\erdp(\alpha)$.
When mapping this RDP consumption to a DP bound, for each $\alpha$ we find $f$ such that $\erdp(\alpha) \leq \erdp^f + \frac{\log(|\Lambda| 2f^2/\delta)}{\alpha-1}$.
The minimum value on the right hand side over all orders is our upper-bound on the privacy loss.

\subsection{Instantiating the Odometer}
\label{sec:odometer-instantiation}

There are two main choices to make when instantiating our privacy odometer.
The first choice is that of $\erdp^f(\alpha)$.
At a given order $\alpha$, the best DP guarantee we can hope for is $\epsilon = \frac{\log(|\Lambda| 2/\delta)}{\alpha-1}$.
We set $\erdp^1(\alpha) = \frac{\log(|\Lambda| 2/\delta)}{\alpha-1}$, ensuring that we can always yield a DP guarantee within a factor two of the lowest possible one when $f=1$.
We then preserve this factor two upper-bound by doubling the filter at every $f$, yielding:
\[
  \erdp^f(\alpha) = 2^{f-1} \frac{\log(|\Lambda| 2/\delta)}{\alpha-1} .
\]
Note that the final DP guarantee will not be a step function doubling at every step, as we can choose the best DP guarantee over all $\alpha$s a posteriori.

We then need to choose the set $\Lambda$ of RDP orders to track, with multiple considerations to balance.
First, the final DP guarantee implied by Equation \ref{eq:dp-odom-curve} has a dependency on $\sqrt{\log(|\Lambda|)}$, so $|\Lambda|$ cannot grow too large.
Second, the largest $\alpha$ will constrain the lowest possible DP guarantee possible.
\cite{10.5555/3157096.3157312} suggests that a minimum DP guarantee of $\edp = \frac{1}{n^2}$, where $n$ is the size of the dataset, is always sufficient. This is because the result of a DP interaction has to be almost independent of the data at such an $\edp$, rendering accounting at lower granularity meaningless.
In this case, setting $\big\{ \Lambda = 2^i, i \in \{1, \log_2(n^2)\} \big\}$ yields the same optimal dependency on $n$ as \cite{10.5555/3157096.3157312}, Theorems 6.1 and 6.3.
In contrast to the result from \cite{10.5555/3157096.3157312}, the resulting odometer does not degrade for values $\edp$ larger $1$, making it more practical in situations where large budgets are expected.

In practice, many applications have high DP cost for which the minimal granularity of $\frac{1}{n^2}$ is too extreme.
Such applications can gain from a higher resolution of RDP orders in the relevant range, yielding a tighter RDP bound.
This effect is particularly important due to our exponential discretisation of the RDP budget spent at each order.
By choosing the range of RDP orders $\Lambda$, we can thus trade-off the $\log(|\Lambda|)$ cost with the RDP order resolution and the minimum budget granularity.
This ease of adaptation, combined with efficient RDP accounting, make our privacy odometer the first to be practical enough to apply to DP deep learning.
In the rest of this paper, we use $\alpha = \{1.25, 1.5, 1.75, \ldots, 9.75, 10, 16, 32\}$ for all results.
These are typical values used in RDP accounting for deep learning, with a slightly higher resolution than the one given as example in \cite{8049725}.

\section{Applications to DP Deep Learning}
\label{sec:applications}

Practical privacy filters and odometers open a new design space for DP, in particular in its applications to deep learning models.
We next showcase early applications of both constructs for classification models on the CIFAR-10 dataset \cite{cifar-dataset}.
Our goal is not to develop state of the art DP models, but to show the promise of adaptive DP budgets in controlled experiments.
Our Opacus based implementation for privacy filters and odometers, as well as the code to replicate experiments, is available at \url{https://github.com/matlecu/adaptive_rdp}.

\begin{figure*}[ht!]
	\centering
	\subfigure[Test Set Accuracy]{
		\includegraphics[width=0.23\textwidth]{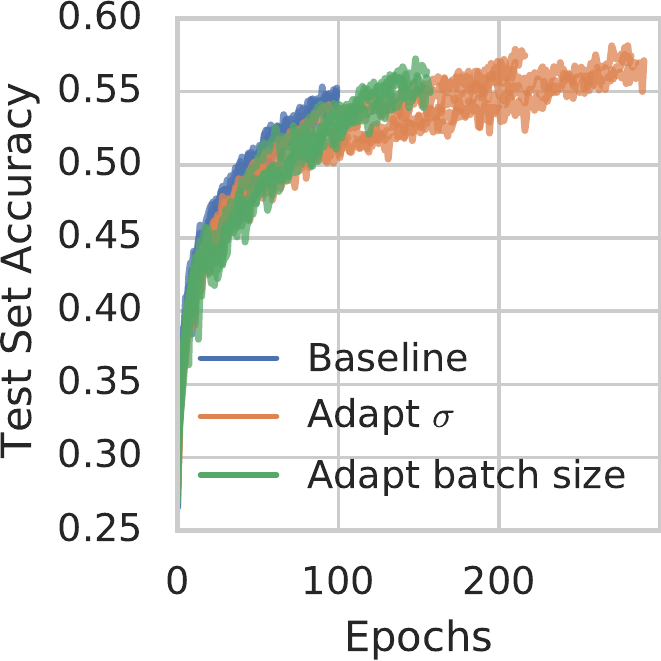}
		\label{fig:adapt-acc}
	}
	\subfigure[Budget Consumed]{
		\includegraphics[width=0.23\textwidth]{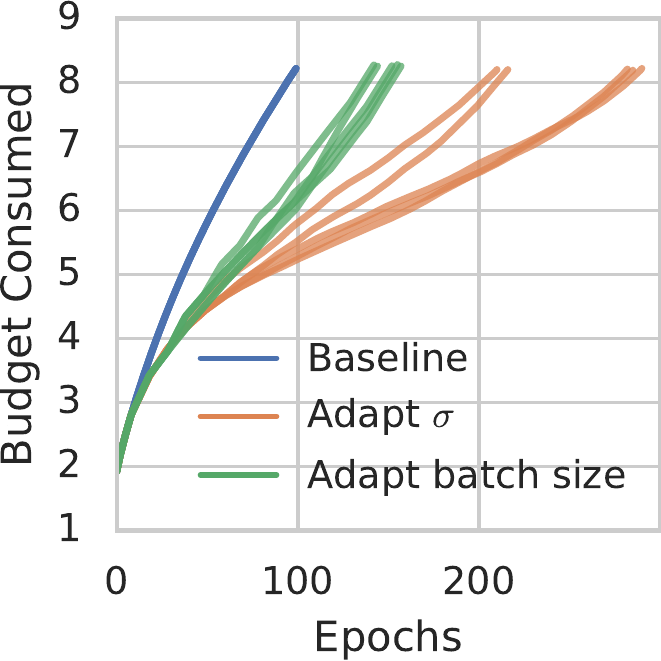}
		\label{fig:adapt-eps}
	}
	\subfigure[Noise Adaptation Example]{
		\includegraphics[width=0.23\textwidth]{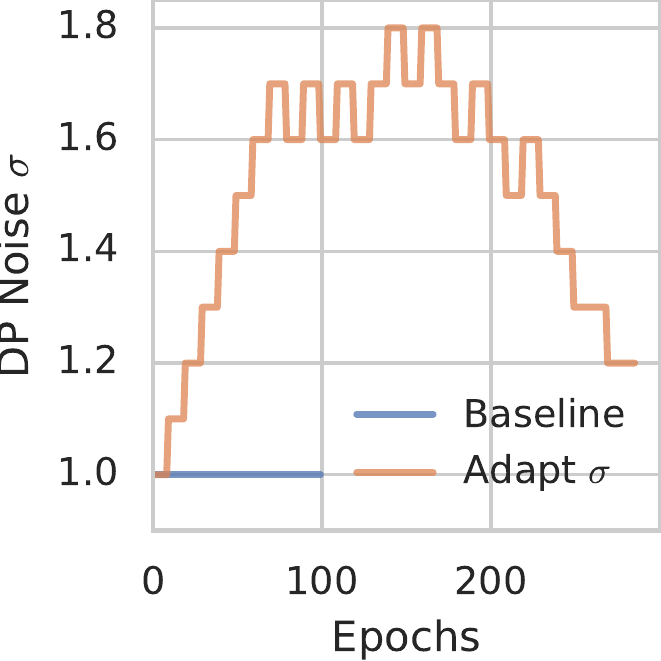}
		\label{fig:adapt-noise-size}
	}
	\subfigure[Batch Adaptation Example]{
		\includegraphics[width=0.23\textwidth]{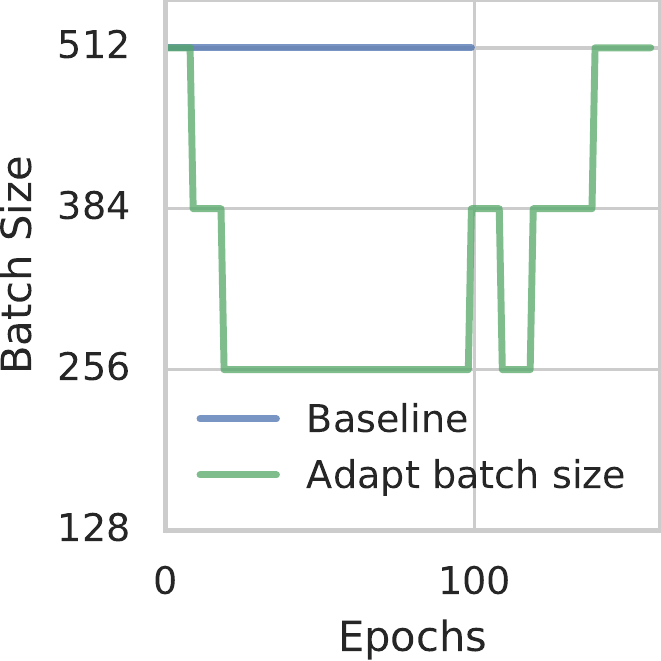}
		\label{fig:adapt-batch-size}
	}
	\vspace{-5pt}
	\caption{
		\textbf{Adaptive strategies}.
    Figures \ref{fig:adapt-acc} and \ref{fig:adapt-eps} respectively show the test set accuracy and DP budget consumed during training, over five independent runs. Adaptive policies can train for longer, for a better average accuracy ($+2.64\%$ for noise adaptation).
    Figures \ref{fig:adapt-noise-size} and \ref{fig:adapt-batch-size} show a representative instance of the evolution of noise and batch size.
	}
	\vspace{-7pt}
	\label{fig:finetune}
\end{figure*}

\subsection{Privacy Filter Applications}
\label{sec:eval:filter}

Intuitively, adaptivity can be used to save budget during training to perform more steps in total, increasing accuracy under a fixed total privacy loss.
By far the most common method to train deep learning models with DP is DP Stochastic Gradient Descent (DP-SGD) \cite{abadi2016deep}, which clips the gradients of individual data points to bound their influence on each SGD step, and adds Gaussian noise to the mini-batch's gradients to ensure that each parameter update is DP. Composition is then used to compute the total privacy loss incurred when training the model.

\paragraph{Adaptation policy.}
Starting from a non adaptive baseline, every $10$ epochs we compute the number of correctly classified train set images.
Such a computation needs to be DP, and we use the Gaussian mechanism with standard deviation $\sigma=100$.
We account for this budget in the total used to train the model.
If the number of correctly classified examples increased by at least three standard deviations, a ``significant increase'' that ensures this change is never due to the DP noise, we additively decrease the privacy budget used at each mini-batch for the next $10$ epochs.
If there isn't such a significant increase, we increase the budget by the same amount, up to that of the non adaptive baseline.
Finally, we only lower the privacy budget when the filter will allow more than $50$ epochs under the current budget.

There are two main parameters in DP-SGD that influence the budget of each mini-batch computation: the size of the batch, and the standard deviation of the noise added to the gradients.
We apply our policy to each parameter separately. For the noise standard deviation, we add or remove $0.1$. For the batch size, we add or remove $128$ data points in the mini-batch (with a minimum of $256$), which is the size of sub-mini-batch that we use to build the final batch.

\paragraph{Results.}
Figure \ref{fig:adaptive-filter} shows these two adaptive policies applied to a baseline ResNet9 model, trained for $100$ epochs with a fixed learning rate of $0.05$, gradient clipping $1$, DP noise $\sigma=1$, and batch size $512$. BatchNorm layers are replaced with LayerNorm to support DP.
As expected, the policy first decreases the privacy budget, increasing the noise (Fig. \ref{fig:adapt-noise-size}) or decreasing the batch size (Fig. \ref{fig:adapt-batch-size}). This translates to less consumption in DP budget (Fig. \ref{fig:adapt-eps}) and longer training, up to $3\times$ the number of epochs (and up to 15h vs 3.5h on one Tesla M60 GPU).
When performance starts to plateau, the DP budget is increased back to its original level.
Despite the slower start in accuracy, depicted on Figure \ref{fig:adapt-acc}, the end result is significantly better. Over $5$ runs, the average accuracy of the baseline is $54.58\%$ (min/max $53.40\%$/$55.24\%$), while adapting the batch size yields $55.56\%$ ($54.93\%$/$56.38\%$) and adapting noise $57.22\%$ ($56.96\%$/$57.4\%$).
The noise adaptation policy always performs better than the best baseline run, by more than $1.72\%$ and with an average increase of $2.64\%$.


\begin{figure*}[ht!]
	\centering
	\subfigure[Fine-tuning]{
		\includegraphics[width=0.23\textwidth]{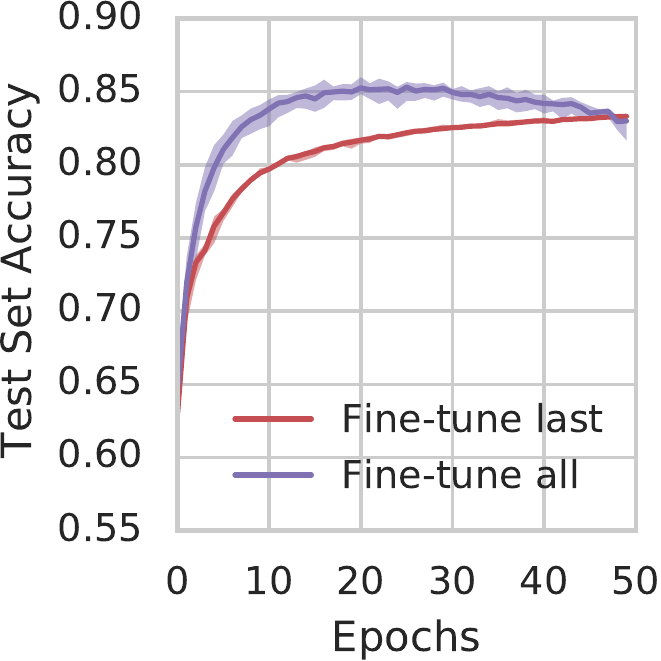}
		\label{fig:finetune-stop-acc}
	}
	\subfigure[Fine-tuning]{
		\includegraphics[width=0.23\textwidth]{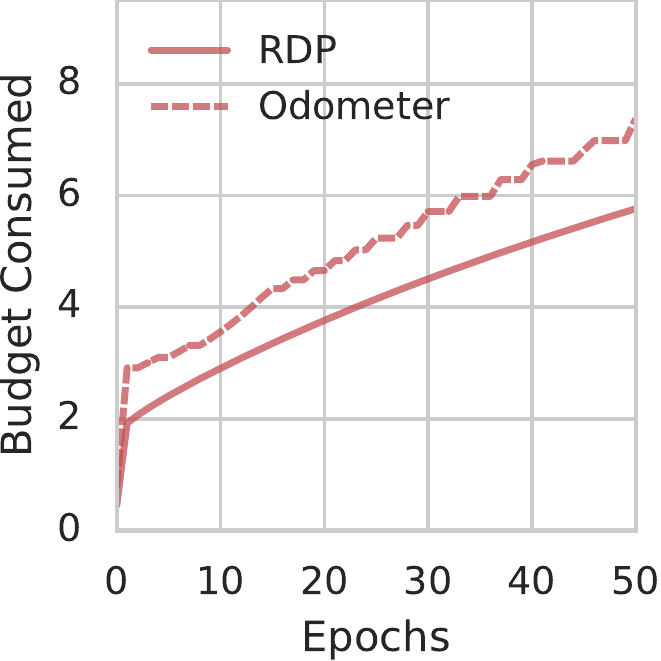}
		\label{fig:finetune-stop-eps}
	}
	\subfigure[Adaptive Fine-tuning]{
		\includegraphics[width=0.23\textwidth]{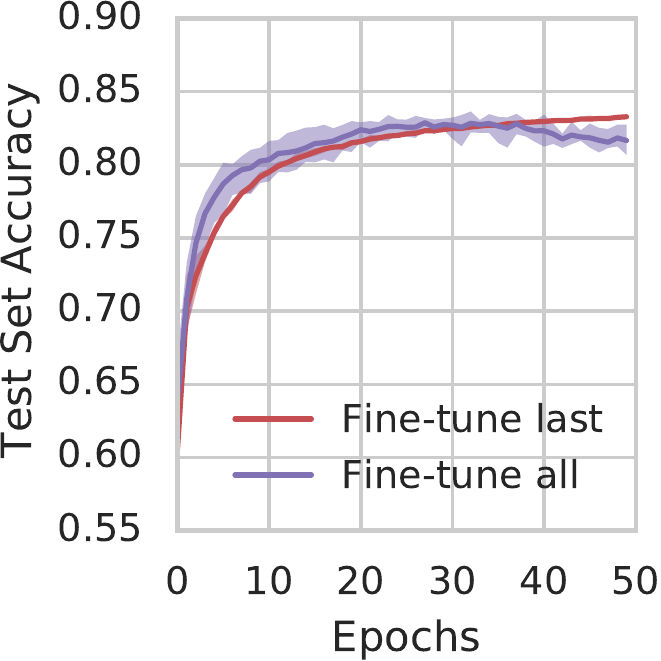}
		\label{fig:finetune-adapt-acc}
	}
	\subfigure[Adaptive Fine-tuning]{
		\includegraphics[width=0.23\textwidth]{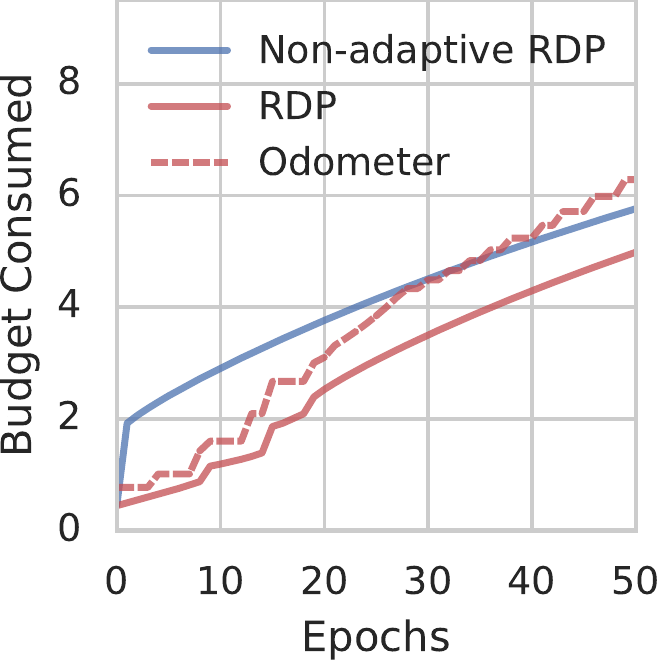}
		\label{fig:finetune-adapt-eps}
	}
	\vspace{-5pt}
	\caption{
    \textbf{Odometer} based privacy loss upper bound (dashed), when fine-tuning an ImageNet model for CIFAR-10. We show the accuracy over 5 runs \ref{fig:finetune-stop-acc} and odometer privacy loss compared the curve of a non-early stopping RDP curve that runs for 50 epochs \ref{fig:finetune-stop-eps}. We then leverage adaptivity to start with lower privacy budgets, and show the resulting accuracy \ref{fig:finetune-adapt-acc} and privacy loss curves \ref{fig:finetune-adapt-eps}.
	}
	\vspace{-7pt}
	\label{fig:adaptive-filter}
\end{figure*}

\subsection{Privacy Odometer Applications}

A natural application of privacy odometers is fine-tuning an existing model with DP, so that the privacy of the fine-tuning dataset in preserved.
In this setup, we expect to reach good performance in a small but unknown number of epochs.
We would like to stop when the model reaches satisfactory accuracy on the new dataset (or stops improving), without incurring more privacy loss than necessary.

\paragraph{Setup.}
We showcase such an adaptive scenario by fine-tuning Pytorch's pre-trained ResNet18 for ImageNet so that it predicts on CIFAR-10. To this end, we modify the last layer of the original model to output only $10$ classes, and resize the images to that they have the expected $224 \times 224$ size.
When fine-tuning the whole model, we still freeze BatchNorm parameters and apply the layer in test mode, because DP-SGD is not compatible with non frozen BatchNorm.
We also show results when training the last layer only.
In each case, we fine-tune the model for $50$ epochs with learning rate $0.01$, DP noise standard-deviation $\sigma=1$, and mini-batch size $512$.
Finally, we combine fine-tuning with an adaptive policy that starts from a larger $\sigma=2$ DP noise, and decreases it by $0.1$ decrements after every epoch during which the number of correctly predicted train set images did not significantly increase (see \S\ref{sec:eval:filter} for details).
Full fine-tuning takes 8h on our Tesla M60 GPU, and last layer fine-tuning 2.5h.

\paragraph{Results.}
Figure \ref{fig:finetune} shows the accuracy reached during fine-tuning either the last layer or the entire model, as well as the DP budget consumption under RDP and the odometer.
We can think of the RDP curve as a privacy filter filling up until reaching the planned $50$ epochs. One can stop at any time, but the entire budget is spent regardless.
Focussing on non-adaptive fine-tuning, \ref{fig:finetune-stop-acc} shows that we may want to use the fully fine-tuned model from epoch $20$, and the end version of the last layer fine-tuned model. Either way, we pay the full $\epsilon=5.76$.
With the odometer, we can stop at the $20^{th}$ epoch having spent only $\epsilon=4.7$.

In case we have an accuracy goal in mind, say $80\%$, we can do even better by using the adaptive policy to fine-tune the model. Under that policy, the fully fine-tuned model consistently reaches this accuracy under $8$ epochs, for an odometer budget of $\epsilon=1.45$, an enormous saving. Without adaptation, the model reaches $80\%$ accuracy in $6$ epochs, but under a higher budget of $\epsilon=3.24$.
We also note that the last layer fine-tuning barely suffers from the adaptive policy, yielding large privacy gains.

\section{Related Work}
The study of Differential Privacy composition under adaptively chosen privacy
parameters was initiated in \cite{10.5555/3157096.3157312}, which formalized
privacy odometers and filters for $(\epsilon, \delta)$-DP and showed
constructions with optimal rate.
The overhead of the privacy filter, as well as the limitations when dealing with
larger DP budgets and the Gaussian mechanism have limited the practical impact
of these results.
Some applications exist, but are further tailored to restricted settings
\cite{ligett2017accuracy}, or using basic composition \cite{sage}.
Our filter and odometer alleviate these limitations.

We also build on RDP \cite{8049725}, a relaxation of pure $\epsilon$-DP which
offers strong semantics and enables tighter accounting of privacy composition
under non adaptive privacy parameters.
Due to its more intuitive and tighter accounting, RDP has become a major
building block in DP libraries for deep learning \cite{opacus,
tensorflow-privacy}.

The closest work to ours is \cite{feldman2020individual}, which studies
adaptive budget composition using RDP in the context of individual privacy
accounting.
While our filter result is identical to theirs, we then leverage it in a fully
adaptive odometer construction.
We use a nested filter design more efficient when packing RDP queries
into the filters, and exponentially growing filters to match the optimial
compostion rate from \cite{10.5555/3157096.3157312}, Theorem 6.1.

\section{Conclusion}

In this paper, we analyse DP composition with adaptively chosen privacy parameters through the lens of RDP.
We show that privacy filters, which allow adaptive privacy budget within a fixed privacy loss upper-bound, do not incur any overhead compared to composition of privacy budgets known a priori.
We leverage this privacy filter to construct a privacy odometer, which gives a running upper-bound on the privacy loss incurred under adaptive privacy budgets, when the sequence of DP computations can stop at any time.
This construction, and its leverage of RDP composition, enables efficient odometers in the large budget regime.

We demonstrate through experiments that these practical privacy filter and odometer enable new DP algorithms, in particular for deep learning.
By adapting the DP budget allocated to each gradient step and stopping training early, we improve accuracy under a fixed privacy loss bound, and reduce the total privacy loss incurred when fine-tuning a model.
We hope that these results will motivate broader applications, from the design of new DP algorithms, to practical parameter tuning and architecture search for DP machine learning models.

\bibliography{bib/paper}
\bibliographystyle{icml2021}

\end{document}